\def\year{2018}\relax
\documentclass[letterpaper]{article} %DO NOT CHANGE THIS
\usepackage{aaai18}  %Required
\usepackage{times}  %Required
\usepackage{helvet}  %Required
\usepackage{courier}  %Required
\usepackage{url}  %Required
\usepackage{graphicx}  %Required
\usepackage{array}
\usepackage{amsmath}
\usepackage{amsthm}
\usepackage{blindtext}
\usepackage{enumitem}
\usepackage{graphicx}
\usepackage{algorithm}% http://ctan.org/pkg/algorithms
\usepackage{algpseudocode}% http://ctan.org/pkg/algorithmicx
\usepackage{amsfonts} 
\usepackage{booktabs}
\usepackage{verbatim}
\newtheorem{theorem}{Theorem}

\newtheorem{lemma}[theorem]{Lemma}
\newtheorem{definition}{Definition}
\newtheorem{remark}{Remark}
\newtheorem{assumption}{Assumption}
\usepackage{color}
\begin{document}
\title{Large Scale Constrained Linear Regression Revisited: Faster Algorithms via Preconditioning\thanks{This research was supported in part by NSF through grants IIS-1422591, CCF-1422324, and CCF-1716400}}
\author{Di Wang  \qquad Jinhui Xu\\
Department of Computer Science and Engineering\\
State University of New York at Buffalo \\
}
\maketitle
\begin{abstract}
In this paper, we revisit the large-scale constrained linear regression problem and propose faster methods based on some recent developments in sketching and optimization. 
Our algorithms combine (accelerated) mini-batch SGD with a new method called two-step preconditioning to achieve an approximate solution  with a time complexity lower than that of 
the state-of-the-art techniques for the low precision case. Our idea can also be extended to the high precision case, which gives an alternative implementation to the Iterative Hessian Sketch (IHS) method with significantly improved time complexity. 
Experiments on benchmark and synthetic datasets suggest that our methods indeed outperform existing ones considerably in both the low and high precision cases. 
\end{abstract}

\section{Introduction}

Since many problems in compressed sensing and machine learning can be formulated as a constrained linear regression problem, such as SVM, LASSO, signal recovery \cite{pilanci2015randomized}, large scale linear regression with constraints now becomes one of the most popular and basic models in Machine Learning and has received a great deal of attentions from both the Machine Learning and Theoretical Computer Science communities. Formally, the problem can be defined as follows,
\begin{equation*}
\min_{x\in{\mathcal{W}}}f(x)={\|Ax-b\|_2^2},
\end{equation*} 
where $A$ is a matrix  in  $\mathbb{R}^{n\times d}$ with $e^d>n>d$ and $\mathcal{W}$ is a closed convex set. The goal is to find an $x\in{\mathcal{W}}$ such that $f(x)\leq (1+\epsilon)\min_{x\in{\mathcal{W}}}f(x)$ or $f(x)- \min_{x\in{\mathcal{W}}}f(x)\leq \epsilon$.\par
On one hand, recent developments on first-order stochastic methods, such as Stochastic Dual Coordinate Ascent (SDCA) \cite{shalev2013stochastic} and Stochastic Variance Reduced Gradient (SVRG) \cite{johnson2013accelerating}, have made significant improvements on the convergence speed of large scale optimization problems in practice.  On the other hand, 
random projection and sampling are commonly used theoretical tools in many optimization problems as preconditioner, dimension reduction or sampling techniques   
to reduce the time complexity. This includes low rank approximation \cite{musco2015randomized}, SVM \cite{paul2013random}, column subset selection   \cite{boutsidis2014near} and $l_p$ regression (for $p\in [1,2]$) \cite{dasgupta2009sampling}. Thus it is very tempting to combine these two types of techniques to develop faster methods with theoretical or statistical guarantee for more constrained optimization problems. Recently, quite a number of works have successfully combined the two types of techniques. For example, 
\cite{gonen2016solving,gonen2015faster} proposed faster methods for Ridge Regression and Empirical Risk Minimization, respectively, by using SVRG, Stochastic Gradient Descent (SGD) and low rank approximation. \cite{zhang2013recovering} achieved guarantee for  Empirical Risk Minimization by using random projection in dual problem. \par

In this paper, we revisit the preconditioning method for solving large-scale constrained linear regression problem, and propose faster algorithms for both the low ($\epsilon\approx10^{-1}\sim 10^{-4}$) and high ($\epsilon\leq 10^{-8}$) precision cases by combining it with some recent developments in sketching and optimization. 
Our main contributions can be summarized as follows.
\begin{table*}[h]
	\begin{center}
		\resizebox{\textwidth}{!}{%
			\begin{tabular}[t]{|c|c|c|c|c|c}
				\hline
				Method & Complexity for Unconstrained &  Complexity for Constrained & Precision\\
				\hline
				\cite{drineas2011faster} & $O\left(nd\log(\frac{d}{\epsilon}) + d^3\log{n}\log{d}+ \frac{d^3\log{n}}{\epsilon}\right)$ & $O\left(nd\log{n}+ \text{poly}(d,\frac{1}{\epsilon^2})\right)$ & Low\\[2ex]
				\hline
				pwSGD\cite{yang2016weighted} & $O\left(nd\log{n} + \frac{d^3\log(\frac{1}{\epsilon})}{\epsilon}\right) $  & $O\left(nd\log{n} + \frac{\text{poly}(d)\log(\frac{1}{\epsilon})}{\epsilon}\right)$ & Low\\[2ex]
				\hline
				\textbf{HDpwBatchSGD} & $O\left(nd\log{n} +  \frac{d^2\log{n}}{\epsilon^2}+ \frac{d^3\log{n}}{r\epsilon^2} \right) $  &  $O\left(nd\log{n}  + \frac{d^2\log{n}}{\epsilon^2}) + \frac{\text{poly}(d)\log{n}}{r\epsilon^2}\right)$ & Low\\[2ex]
				\hline
				\textbf{HDpwAccBatchSGD} & $O\left(nd\log{n} +  \frac{d^2\log{n}}{\epsilon}+ \frac{d^3\log{n}}{r\epsilon}+rd^2\log \frac{1}{\epsilon}  \right)$  &  $O\left(nd\log{n}  + \frac{d^2\log{n}}{\epsilon} + \frac{\text{poly}(d)\log{n}}{r\epsilon}+r\text{poly}(d)\log \frac{1}{\epsilon}\right)$ & Low\\[2ex]
                \hline				
				\cite{rokhlin2008fast,avron2010blendenpik} & $O(nd\log{\frac{d}{\epsilon}}+d^3\log d)$  & --- & High\\  [2ex]
				\hline
				IHS \cite{pilanci2016iterative} & $O\left(nd\log{d}\log\frac{1}{\epsilon}+d^3\log\frac{1}{\epsilon}\right)$ & $O\left((nd\log d+\text{poly}(d))\log{\frac{1}{\epsilon}}\right)$ & High\\[2ex]
				\hline
				{Preconditioning+SVRG}& $O(nd\log d + (nd+d^3)\log\frac{1}{\epsilon})$ &$O(nd\log d + (nd+\text{poly}(d))\log\frac{1}{\epsilon})$ & High\\ [2ex]
				\hline
				\textbf{pwGradient}& $O(nd\log d + (nd+d^3)\log\frac{1}{\epsilon})$ &$O(nd\log d + (nd+\text{poly}(d))\log\frac{1}{\epsilon})$ & High\\ [2ex]
				\hline
		\end{tabular}}
	\end{center}
	\caption{Summary of the time complexity of several linear regression methods  for  finding $x_{t}$ such that $||Ax_t-b||^2_2-||Ax^*-b||^2_2\leq \epsilon$. For sketching based methods, we use the Subsampled Randomized Hadamard Transform (SRHT) \cite{tropp2011improved} as the sketch matrix. All methods run in sequential environment.  $r$ is an input in our method. '---' means not applicable.}
\end{table*}
\begin{itemize}
	\item For the low precision case, we first propose a novel algorithm called HDpwBatchSGD ({\em i.e.,} Algorithm 2) by combining a new method called two step preconditioning with mini-batch SGD.
	  Mini-batch SGD is a popular way for improving the efficiency of SGD.  It uses several samples, instead of one, in each iteration and runs the method on all these samples (simultaneously).  Ideally, we would hope for a factor of $r$ speed-up on the convergence if using a batch of  size $r$. However, this is not always possible for general case. Actually in some cases, there is even no speed-up at all when a large-size batch is used \cite{takac2013mini,byrd2012sample,dekel2012optimal}.  A unique feature of our method is its optimal speeding-up with respect to  the batch size, {\em i.e.} the iteration complexity will decrease by a factor of $b$ if we increase the batch size by a factor of  $b$. To further improve the running time, we also use 
	  %\textcolor{red}{Later, based on 
	  the Multi-epoch Stochastic Accelerated mini-batch SGD \cite{ghadimi2013optimal} to obtain another slightly different algorithm
	   %we propose an algorithm 
	   called HDpwAccBatchSGD, which has a time complexity lower than that of the state-of-the-art technique \cite{yang2016weighted}. 
	\item The optimality on speeding-up  further inspires us to think about how it will perform if using the whole gradient, {\em i.e.} projected Gradient Descent,  (called pwGradient ({\em i.e.,} Algorithm 4)).  A somewhat surprising discovery is that it actually allows us to have an alternative implementation of the Iterative Hessian Sketch (IHS) method, which is arguably the state-of-the-art technique for the high precision case. Particularly, we are able to show that one step of sketching is sufficient for IHS, instead of a sequence of sketchings used in the current form of IHS. This enables us to considerably improve the time complexity of IHS. 
	\item Numerical experiments on large synthetic/real benchmark datasets confirm the theoretical analysis of HDpwBatchSGD and pwGradient.  Our methods outperform existing ones in both low and high precision cases.  

\end{itemize}

\section{Related Work}

 There is a vast number of papers studying the large scale constrained linear regression problem from different perspectives \cite{yang2016implementing}.
 We mainly focus on those results that have theoretical time complexity guarantees (note that the time complexity cannot depend on the condition number of $A$), due to their similar natures to ours. We summarize these methods in \textbf{Table 1}.\par

For the low precision case, 
\cite{drineas2011faster}  directly uses sketching with a very large sketch size  of $\text{poly}(\frac{1}{\epsilon^2})$, which is difficult to determine the optimal sketch size in practice (later, we show that our proposed method avoids this issue). 
% is robust to the sketch size.} 
The state-of-the-art technique is probably the one in \cite{yang2016weighted},  which presents an algorithm for solving the general $l_p$ regression problem and shares with ours the first step of preconditioning. 
Their method then applies the weighted sampling technique, based on the leverage score and SGD, while ours first conducts a further  step of preconditioning, using uniform sampling in each iteration, and then applies mini-batch SGD. Although their paper mentioned the mini-batch version of their algorithm, there is no theoretical guarantee on the quality and convergence rate, while our method provides both and runs faster in practice. Also we have to note that, even if we set $r=1$ in HDpwAccBatchSGD, the time complexity is less than that in pwSGD when $\frac{1}{\epsilon}>n$. \cite{needell2016batched} also uses mini-batch SGD to solve the linear regression problem.  Their method is based on importance sampling, while ours uses the much simpler uniform sampling; furthermore, their convergence rate heavily depends on the condition number and batch partition of $A$, which means that there is no fixed theoretical guarantee for all instances. \par
For the high precision case, 
unlike the approach in \cite{rokhlin2008fast},  our method can be extended to the constrained case. Compared with IHS \cite{pilanci2016iterative}, ours uses only one step of sketching and thus has a lower time complexity. Although  a similar time complexity can be achieved by using the preconditioning method in \cite{yang2016implementing} and SVRG, 
ours performs better in practice.
We notice that \cite{pmlr-v70-tang17a} has recently  studied the large scale linear regression with constraints via (Accelerated) Gradient Projection and IHS. But there is no guarantee on the time complexity, and it is also unclear how to choose the best parameters. For these reasons, we do not compare it with ours here.\par

\section{Preliminaries}

Let $A$ be a matrix  in  $\mathbb{R}^{n\times d}$ with $e^d> n>d$ and $d=\text{rank}(A)$,  and denote $A_{i}$  and $A^{j}$  be  its i-th row ({\em  i.e.,}  $A_{i}\in\mathbb{R}^{1\times d}$) and j-th column, respectively.  (Note that our proposed methods can be easily extended to the case of $d>\text{rank}(A)$.)
 Let $\|A\|_2$ and $ \| A\|_F$ be the spectral norm and Frobenius norm of $A$, respectively, and $\sigma_{\text{min}}(A)$ be the minimal singular value of $A$. In this section we will give several definitions and lemmas that will be used throughout this paper. Due to space  limit, we leave all the proofs to the full  paper.

\subsection{Randomized Linear Algebra}

First we give the definition of ($\alpha$, $\beta$, 2)-conditioned matrix. Note that it is a special case of  ($\alpha$, $\beta$, $p$)-well conditioned basis in \cite{yang2016implementing}. 

\begin{definition}{(($\alpha$,$\beta$,2)-conditioned) \cite{yang2016weighted})}
 A matrix	$U \in \mathbb{R}^{n\times d}$ is  called ($\alpha$,$\beta$,2)-conditioned  if  $ \| U\|_{F} \leq \alpha$ and for all $x\in\mathbb{R}^{d\times 1}$, $\beta\|Ux\|_2\geq\|x\|_2$, i.e $\sigma_{\text{min}}(U)\geq \frac{1}{\beta}$.
\end{definition}

Note that if $U$ is an orthogonal matrix, it is ($\sqrt{d}$,1,2)-conditioned.  Thus we can view an ($\alpha$,$\beta$,2)-conditioned matrix as a generalized orthogonal matrix.
The purpose of introducing the concept of ($\alpha$,$\beta$,2)-conditioned is for obtaining in much less time a matrix which can approximate the orthogonal basis of matrix $A$. 
Clearly,  if we directly calculate the orthogonal basis of $A$, it will take $O(nd^2)$ time.  However we can get an $(O(\sqrt{d}),O(1),2)$-conditioned matrix $U$ from $A$ (called $O(\sqrt{d}),O(1),2)$-conditioned basis of $A$) in $o(nd^{2})$ time, through \textbf{Algorithm 1}. Note that in practice we can just set $O(1)$ as a small constant and return $R$ instead of $AR^{-1}$.\par

\begin{algorithm}\label{alg:1}
	\caption{Constructing ($O(\sqrt{d})$,$O(1)$,2)-conditioned basis of $A$}
	\begin{algorithmic}[1]
		\State Construct an oblivious subspace embedding(sketch matrix) $S\in \mathbb{R}^{s\times n}$ with $n>s>d$ that satisfies the following condition with high probability, $\forall  x\in \mathbb{R}^d$,
		\begin{equation*}
		(1-O(1))\|Ax\|_2\leq \|SAx\|_2 \leq (1+O(1))\|Ax\|_2,
		\end{equation*}
		\State [Q,R]=QR-decomposition(SA),  where $Q\in\mathbb{R}^{s\times d}$ is an orthogonal matrix.  Then  $AR^{-1}$  is an $O(\sqrt{d})$,$O(1)$,2)-conditioned basis of $A$.\\
		\Return $AR^{-1}$ or $R$
		
	\end{algorithmic}
\end{algorithm}
Next, we give the definition of Randomized Hadamard Transform \cite{tropp2011improved}, which is the tool to be used in the second step of our preconditioning.\\

\begin{definition}{(Randomized Hadamard Transform)}
	$M=HD \in\mathbb{R}^{n\times n}$ is called a Randomized Hadamard Tranform, where  $n$ is assumed to be $2^s$ for some integer $s$, 
	$D\in\mathbb{R}^{n\times n}$ is a diagonal Rademacher matrix (that is, each $D_{ii}$ is drawn independently from $\{1,-1\}$ with probability 1/2), 
	and $H\in \mathbb{R}^{n\times n}$ is an $n\times n$ Walsh-Hadamard Matrix  scaled by a factor of $1/\sqrt{n}$, {\em i.e.,}
	$$
	H=\frac{1}{\sqrt{n}}H_n,
	H_n=\begin{pmatrix}
	H_{\frac{n}{2}}  &  H_{\frac{n}{2}}\\
	H_{\frac{n}{2}} & - H_{\frac{n}{2}} \\
	\end{pmatrix}, 
	H_2=\begin{pmatrix}
	1 & 1\\
	1 & -1\\
	\end{pmatrix}.$$
\end{definition}
Randomized Hadamard Transform has two important features. One is that it takes only $O(n \log n)$ time to multiply a vector, and the other is that it can ``spread out'' orthogonal matrices.

Since an ($\alpha$,$\beta$,2)-conditioned matrix can be viewed as an approximate orthogonal matrix, an immediate question is whether an  ($\alpha$,$\beta$,2)-conditioned matrix can also achieve the same result. We answer this question by the following theorem,

\begin{theorem}\label{thm1}
	Let $HD$ be a Randomized Hadamard Transform, and $U\in \mathbb{R}^{n\times d}$ be an ($\alpha$,$\beta$,2)-conditioned matrix.  Then, the following holds for any constant $c >1$:
	\begin{equation}\label{eq:1}
	Pr\{\max_{i=1,2,\dots,n}\|(HDU)_i\|_2 \geq (1+\sqrt{8\log(cn)})\frac{\alpha}{\sqrt{n}}\} \leq \frac{1}{c}.
	\end{equation}
\end{theorem}

By the above theorem, we can make each row of $HDU$ have no more than one value with high probability.  Since $\alpha=O(\sqrt{d})$,  the norm of each row is small.  Also since $H,D$ are orthogonal,  we have $\|HDUy-HDb\|_2=\|Uy-b\|_2$ for any $y$. \par

\subsection{Strongly Smooth SGD and Mini-batch SGD}

Consider the following general case of a convex optimization problem: $\min_{x\in \mathcal{W}} F(x)=\mathbb{E}_{i \sim D}f_i(x)$,
where $i$ is drawn from the  distribution of $D=\{ p_i\}_{i=1}^{n}$  and $\mathcal{W}$ is a closed convex set.  We assume the following.
\begin{assumption}
	$F(\cdot)$ is $L$-Lipschitz. That is, for any $x,y\in \mathcal{W}$,
	\begin{equation*}
	\|\nabla{F}(x)-\nabla{F}(y)\|_2 \leq L\|x-y\|_2.
	\end{equation*}
\end{assumption}
\begin{assumption}
	F(x) has strong convexity parameter $\mu$, {\em i.e.},
	\begin{equation*}
	\langle x-y,\nabla F(x)- \nabla F(y) \rangle \geq \mu \|x-y\|_2^2 , \forall x,y \in \mathcal{W}.
	\end{equation*}
\end{assumption}
%\end{enumerate}
Now let the Stochastic Gradient Descent (SGD) update in the $(k+1)$-th iteration be 
\begin{align}\label{eq:2}
%x_{k+1}=P_{\mathcal{W}}(x_k - \eta_t\nabla f_{i_k}(x_k)),
x_{k+1}&=\arg \min_{x\in \mathcal{W}}\eta_k\langle \nabla f_{i_k}(x_k),x-x_k\rangle+\frac{1}{2}\|x-x_{k}\|_2^2\\
&=P_{\mathcal{W}}(x_k - \eta_k\nabla f_{i_k}(x_k))
\end{align}
where $i_k$ is drawn from the distribution $D$, $x_0$ is the initial number, and $P_{\mathcal{W}}$ is the projection operator.
If we denote
\begin{equation*}
x^*=\arg\min_{x\in{\mathcal{W}}}{F(x)},  \sigma^2=\sup_{x\in \mathcal{W}}\mathbb{E}_{i\sim D}\|\nabla f_i(x)-\nabla F(x)\|_2^{2},
\end{equation*}
then we have the following theorem, given in \cite{lan2012optimal}.

\begin{theorem}\label{thm2}
	If Assumption 1 hold, after $T$ iterations of the SGD 
	iterations of (\ref{eq:2}) with fixed step-size
	\begin{equation}\label{eq:4}
	\eta=\min(\frac{1}{2L},\sqrt{\frac{D_{\mathcal{W}}^2}{2T\sigma^2}}),
	\end{equation}
	where $D_{\mathcal{W}}=\sqrt{\max_{x\in\mathcal{W}}\frac{1}{2}\|x\|_2^2-\min_{x\in \mathcal{W}}\frac{1}{2}\|x\|_2^2}$. Then the 
	inequality $\mathbb{E}F(x_T^{\text{avg}})-F(x^*)\leq \frac{3\sqrt{2}D_{\mathcal{W}}\sigma}{\sqrt{T}}$ is true, where $x_T^{\text{avg}}=\frac{\sum_{i=1}^{T}x_i}{T}$. Which means after
	\begin{equation}\label{eq:3}
     T= \Theta(\frac{D_{\mathcal{W}}^2\sigma^2}{\epsilon^2})
     \end{equation} iterations ,we have $\mathbb{E}F(x_T^{\text{avg}})-F(x^*)\leq \epsilon$.
\end{theorem}

Instead of sampling one term in each iteration, mini-batch SGD samples several terms in each iteration and takes the average. Below we consider the uniform sampling version,
\begin{equation}\label{eq:5}
\min_{x\in\mathcal{W}}F(x)=\frac{1}{n}\sum_{i=1}^{n}{f_i(x)}.
\end{equation}
Note that for a mini-batch of size $r$, let $\tau$ denote 
%if we denote $\tau$ as 
the sampled indices and $g_{\tau}=\frac{\sum_{i\in \tau}\nabla f_i(x)}{r}$, where each index in $\tau$ is i.i.d uniformly sampled.  Then, we have $\sigma^2_{batch}=\sup_{x\in \mathcal{W}}\mathbb{E}_{\tau}\|g_{\tau}-\nabla F(x)\|_2^2\leq \frac{\sigma^2}{r}$. This means that the variance can be reduced by  a factor of $r$  if we use a sample of size $r$.

\begin{remark}
	Note that our mini-batch sampling strategy is different from that in \cite{needell2016batched}, which  is to partition all the indices into $\lceil\frac{n}{r}\rceil$ groups and  samples only within one group in each iteration.
\end{remark}

\section{Two-step Preconditioning Mini-batch SGD}
\subsection{Main Idea}
The  idea of our algorithm is to use two steps of preconditioning to reform the problem in the following way,
\begin{align}\label{eq:7}
&\min_{x\in{\mathcal{W}}}f(x)={\|Ax-b\|_2^2}=\min_{y\in {\mathcal{W'}}}\|Uy-b\|_2^2\\
&=\|HDUy-HDb\|_2^2=\frac{1}{n}\sum_{i=1}^{n}{n\|(HDU)_{i}y-(HDb)_i\|_2^2}.
\end{align}
The first step of the preconditioning (7) is  to get $U$, an ($O(\sqrt{d})$,$O(1)$,2)-conditioned basis of $A$ ({\em i.e.,} $U=AR^{-1}$; see Algorithm 1), which means that the function in problem (7) is an $O(d)$-smooth (actually it is $O(1)$-smooth, see Table 2) and $O(1)$-strongly convex function. The second step of the preconditioning (8) is to use Randomized Hadamard Transform to `spread out' the row norm of $U$ by \textbf{Theorem \ref{thm1}}. 
Then, we use mini-batch SGD with uniform sampling for each iteration. We can show that  $x^*=R^{-1}y^*$, where $y^*=\arg\min_{y\in {\mathcal{W'}}}{\|HDUy-HDb\|_2^2}$, and $\mathcal{W'}$ is the convex set corresponding to $\mathcal{W}$.

\subsection{Algorithm}
The main steps of our algorithm are given in the following {\bf Algorithm 2}.
\begin{algorithm}
	\caption{HDpwBatchSGD($A$,$b$,$x_0$,$T$,$r$,$\eta$,$s$)} 
	$\mathbf{Input}$: 	 $x_0$ is the initial point, $r$ is the batch size, $\eta$ is the fixed step size, $s $ is the sketch size, and $T$ is the iteration number.
	\begin{algorithmic}[1]
		\State Compute $R\in \mathbb{R}^{d\times d}$ which makes $AR^{-1}$ an $(O(\sqrt{d}),O(1),2)$-conditioned basis of $A$ as in Algorithm 1 by using a sketch matrix $S$ with size $s \times n$.
		\State  Compute $HDA$ and $HDb$, where $HD$ is a Randomized Hadamard Transform.
		\For {$t \leftarrow 1,\dots T$}{
		\State Randomly sample  an indices set $\tau_{t}$ of size $r$, where each index in $\tau_{t}$ is i.i.d uniformly sampled.
		\State 
		$c_{\tau_{t}}=\frac{2n}{r}\sum_{j \in \tau_t}{(HDA)_j^T[(HDA)_jx_{t-1}-(HDb)_j]} 
		=\frac{2n}{r}(HDA)_{\tau_t}^T[(HDA)_{\tau_{t}}x_{t-1}-(HDb)_{\tau_{t}}]  $  
		\State$
		x_{t}=\arg\min_{x\in \mathcal{W}}{ \frac{1}{2}\|R(x_{t-1}-x)\|^2_{2}+\eta\langle c_{\tau_{t}},x \rangle}
		=\mathcal{P}_{\mathcal{W}}(x_{t-1}-\eta R^{-1}(R^{-1})^Tc_{\tau_{t}})$
		\EndFor}\\
		\Return $x_T^{\text{avg}}=\frac{\sum_{i=1}^{T}x_i}{T}$
	\end{algorithmic}
\end{algorithm}

Note that the way of updating $x_{t}$ in {\bf Algorithm 2} is equivalent to the updating procedure  of $y_{t}$  for the reformed problem (8) ({\em i.e.,} set $y_0=R^{-1}x_0$, then use mini-batch SGD and let $x_{t}=R^{-1}y_{t}$).  There are several benefits if we update $x_{t}$ directly. 
\begin{itemize}
	\item Directly updating $y_{t}$ needs additional $O(nd^2)$ time since we have to compute $AR^{-1}=U$, while  updating $x_{t+1}$ can avoid that, {\em i.e.}, it is sufficient to just compute $R$.
	\item  In practice, the domain set $\mathcal{W}$ of $x$ is much more regular than the domain set of $y$, {\em i.e.,} $\mathcal{W'}$ in (\ref{eq:7}). Thus it is much easier to solve the optimization problem in Step 7.
\end{itemize}

In the above algorithm, Step 1  is the same as the first step of pwSGD in \cite{yang2016weighted}.  But the later steps are quite different.  Particularly, our algorithm does not need to estimate the approximate leverage score of $U$ for computing the sampling probability in each iteration. It uses the much simpler uniform sampling, instead of the weighted sampling.  By doing so, we need to compute $HDA$ and $HDb$, which takes only $O(nd\log{n})$ time, is much  faster than the $O(nd^2)$ time required for exactly computing the leverage score of $U$,  and costs approximately the same time  ({\em i.e.,} $O(\text{nnz}(A)\log{n})$) for  computing the approximate leverage score. The output is also different as ours is the average of $\{x_i\}_{i=1}^{T}$. Also, we note that in the experiment section, \cite{yang2016weighted} uses the exact leverage score instead of its approximation. By Theorem \ref{thm1} and \ref{thm2}, we can get an upper bound on $\sigma^2$ and our main result (note that $\sup_{x\in\mathcal{W}}\|Ax-b\|_2^2$ in the result is determined by the structure of the original problem and thus is assumed to be  a constant here).
\begin{theorem}\label{thm4}
	Let  $A$ be a matrix in $\mathbb{R}^{n\times d}$, $r$ be the batch size and $b$ be a vector in $\mathbb{R}^{d}$.  Let $f(x)$ denote $\|Ax-b\|^2_2$. Then with some fixed step size $\eta$ in (\ref{eq:4}), we have 
	\begin{equation}
	\mathbb{E}{f(x_T^{\text{avg}})-f(x^*)}\leq \frac{3\sqrt{2}D_{\mathcal{W}}\sigma}{\sqrt{rT}},
	\end{equation}
	where $\sigma^2=O(d\log(n)\sup_{x\in\mathcal{W}}\|Ax-b\|_2^2)$ with high probability.
	That is, after $T= \Theta(\frac{d\log{n}}{r\epsilon^2})$ iterations, {\bf Algorithm 2} ensures the following with high probability, $\mathbb{E}|{f(x_T^{\text{avg}})-f(x^*)}|\leq \epsilon$. 
\end{theorem}
The time complexity of our algorithm can be easily obtained as $$\text{time}(R)+O(nd\log n+\text{time}_{\text{update}}\frac{d\log{n}}{r\epsilon^2}),$$
where $\text{time}(R)$ is the time for computing $R$ in Step 1. Different sketch matrices and their time complexities for getting $R$  are shown in \textbf{Table 2}. Step 2 takes $O(nd\log n)$ time.  $\text{time}_{\text{update}}$ is the time for updating $x_{t+1}$ in Steps 5 and  6. Step 5 takes $O(rd)$ time, while Step 6 takes $\text{poly}(d)$ time since it is just a quadratic optimization problem in $d$ dimensions.  Thus, if we use SRHT as the sketching matrix $S$, the overall time complexity of our algorithm is
\begin{equation}\label{eq:12}
O(nd\log n+d^3\log{d}+(\text{poly}(d)+rd)\frac{d\log n}{r\epsilon^2}).
\end{equation}
\begin{table}[h]
	\caption{Time complexity for computing $R$ in step 1 of Algorithm 2,  3, 4 with different sketch matrix \cite{yang2016weighted}.}
	\centering\small
	\begin{tabular*}{\linewidth}{@{\extracolsep{\fill}}p{0.25\linewidth}p{0.45\linewidth}p{0.2\linewidth}@{}}
		\toprule
		Sketch Matrix & Time Complexity & $\kappa(AR^{-1})$\\
		\midrule
		Gaussian Matrix & $O(nd^2)$ & $O(1)$\\
		\midrule
		SRHT\cite{pilanci2016iterative} & $O(nd\log d +d^3 \log d)$ & $O(1)$\\
		\midrule
		CountSketch & $O(\text{nnz}(A)+d^4)$ &$O(1)$\\
		\midrule
		Sparse $l_2$ Embedding & $O(\text{nnz}(A)\log d+d^3\log d)$ &$O(1)$ \\
		\bottomrule
	\end{tabular*}
\end{table}
\subsection{Further Reducing the Iteration Complexity}
Theorem \ref{thm4} does not make use of the properties of $O(1)$-strongly convexity and condition number $\frac{L}{\mu}=O(1)$ of the problem after the two-step preconditioning. 
We can apply a different first-order method to achieve an $\epsilon$-error in $\Theta(\frac{d\log{n}}{r\epsilon}+\log(\frac{1}{\epsilon}))$ iterations, instead of $\Theta(\frac{d\log{n}}{r\epsilon^2})$ iterations as in Theorem \ref{thm4}. The preconditioning steps are the same, and the optimization method is the multi-epoch stochastic accelerated gradient descent, which was proposed in \cite{ghadimi2012optimal,ghadimi2013optimal}. In stead of using Theorem \ref{thm2}, we will use the following theorem whose proof was given in \cite{ghadimi2013optimal}. 

\begin{theorem}
	If Assumption 1 and 2 hold and $\epsilon<V_0$, then after $O(\sqrt{\frac{L}{\mu}}\log(\frac{V_0}{\epsilon})+\frac{\sigma^2}{\mu\epsilon})$ iterations of stochastic accelerated gradient descent with $O(\log(\frac{V_0}{\epsilon}))$ epochs, the output of multi-epoch stochastic accelerated gradient descent $x_S$ satisfies $\mathbb{E}F(x_S)-F(x^*)\leq \epsilon$, where $V_0$ is a given bound such that $F(x_0)-F(x^*)\leq V_0$.
\end{theorem}

Thus,  we can use a two-step preconditioning and multi-epoch stochastic accelerated mini-batch gradient descent to obtain an algorithm (called HDpwAccBatchSGD) similar 
to {\bf Algorithm 2},
 %(we call this algorithm HDpwAccBatchSGD), 
as well as the following theorem. Due to the space limit, we leave the details of the algorithm to the full paper.
\begin{theorem}
	Let  $A$ be a matrix in $\mathbb{R}^{n\times d}$, $r$ be the batch size and $b$ be a vector in $\mathbb{R}^{d}$.  Let $f(x)$ denote $\|Ax-b\|^2_2$, and fix $\epsilon<V_0$. Then with high probability, 
		  after $O(\log(\frac{V_0}{\epsilon})+\frac{d\log n}{r\epsilon})$ iterations of stochastic accelerated gradient descent with $S=O(\log(\frac{V_0}{\epsilon}))$ epochs of HDpwAccBatchSGD, the output $x_S$ satisfies $\mathbb{E}F(x_S)-F(x^*)\leq \epsilon$. Moreover, if we take SRHT as the sketching matrix , the total time complexity is 
    \begin{multline}\label{eq:13}
	O( nd\log{n} + \frac{d^2\log{n}}{\epsilon}
    +\frac{\text{poly}(d)\log{n}}{r\epsilon}+r\text{poly}(d)\log \frac{1}{\epsilon}).
	\end{multline}
\end{theorem}

\section{Improved Iterative Hessian Sketch}

Now, we go back to our results in (\ref{eq:12}) and (\ref{eq:13}). One benefit of these results is that  $\epsilon$ is independent of $n$ and depends only on $\text{poly}(d)$ and $\log n$. If directly using the Variance Reduced methods developed in recent years (such as \cite{johnson2013accelerating}), we can get the time complexity $O((n+\kappa)\text{poly}(d)\log\frac{1}{\epsilon})$. Comparing with these methods, we know that HDpwBatchSGD and HDpwAccBatchSGD are more suitable  for the \textbf{low precision and large scale} case. 
Recently \cite{pilanci2016iterative} introduced the Iterative Hessian Sketch (IHS) method to solve the large scale constrained linear regression problem (see Algorithm 3). 
IHS is capable of achieving high precision, but needs a sequence of sketch matrices $\{S^{t}\}$ (which seems to be unavoidable due to their analysis) to ensure the linear convergence with high probability. Ideally, if we could use just one sketch matrix,  it would greatly reduce the running time. In this section, we show that by adopting our preconditioning strategy, it is indeed possible to use only one sketch matrix in IHS to achieve the desired linear convergence with high probability (see pwGradient(\textbf{Algorithm 4})).\par

Our pwGradient algorithm  uses the first step of preconditioning ({\em i.e.,} Step 1 of Algorithm 2) and then performs gradient decent (GD) operations, instead of the mini-batch SGD in Algorithm 2 (note that we do not need the second step of preconditioning since it is an orthogonal matrix). Since the condition number after preconditioning is $O(1)$ (see \textbf{Table 2}), by the convergence rate of GD, we know that only $O(\log \frac{1}{\epsilon})$ iterations are needed to attain $\epsilon$-solution.  %formally,
\begin{algorithm}[h]
	\caption{IHS($A$,$b$,$x_0$,$s$)\cite{pilanci2015randomized}} 
	$\mathbf{Input}$: 	 $x_0$ is the initial point, and $s $ is the sketch size.
	\begin{algorithmic}[1]
		\For{$t=0,1,\cdots,T-1$}{
			\State Generate an independent sketch matrix $S^{t+1}\in \mathbb{R}^{s\times n}$ as in step 1 of Algorithm 1. Compute $M=S^{t+1}A$.
			\State  Perform the updating 
			\begin{align*}
			x_{t+1}&=\arg \min_{x\in \mathcal{W}}{\frac{1}{2}\|M(x-x_t)\|_2^2+\langle A^T(Ax_t-b), x\rangle}\\
			&=\mathcal{P}_{\mathcal{W}}(x_{t}-M^{-1}(M^{-1})^TA^T(Ax_t-b))
			\end{align*}
			\EndFor}\\
		\Return $x_T$
	\end{algorithmic}
\end{algorithm}
\vspace{-0.1in}
\begin{algorithm}[h]
	\caption{pwGradient($A$,$b$,$x_0$,$s$, $\eta$)} 
	$\mathbf{Input}$: 	 $x_0$ is the initial point, $s $ is the sketch size, and $\eta$ is the step size.
	\begin{algorithmic}[1]
		\State Compute $R\in \mathbb{R}^{d\times d}$ which makes $AR^{-1}$ an $O(\sqrt{d}),O(1),2)$-conditioned basis of $A$ as in Algorithm 1 by using a sketch matrix $S$ with size $s \times n$.
		\For{$t=0,1,\cdots,T-1$}{
			\State  Perform the updating
			\begin{align*}
			x_{t+1}&=\arg \min_{x\in \mathcal{W}}{\frac{1}{2}\|R(x-x_t)\|_2^2+\eta\langle 2A^T(Ax_t-b), x\rangle}\\
			&=\mathcal{P}_{\mathcal{W}}(x_{t}-2\eta R^{-1}(R^{-1})^TA^T(Ax_t-b))
			\end{align*}
			\EndFor}\\
		\Return $x_T$
	\end{algorithmic}
\end{algorithm}

\begin{theorem}
	Let $f(x)=\|Ax-b\|_2^2$.  Then, for some step size $\eta=O(1)$ in pwGradient, the following holds,  
	\begin{equation}
	f(x_t)-f(x^*)\leq (1-O(1))^tO((f(x_0)-f(x^*)))
	\end{equation}
	with high probability.
\end{theorem}

 Below we will reveal the relationship between IHS and pwGradient. Particularly, we will show that when $\eta=\frac{1}{2}$, the updating in pwGradient with sketching matrix $S$ is equivalent to that of IHS with $\{S^{t}\}=S$.  Let $QR$ be the QR-decomposition of $S^{t+1}A=SA$. Then, we have 
\begin{align*}
(QR)^{-1}((QR)^{-1})^T=
(R)^{-1}(R^{-1})^T.
\end{align*}\par
 Although they look like the same, the ideas behind them are quite different. IHS is based on sketching the Hessian and uses the second-order method of the optimization problem,  while ours is based on preconditioning the original problem and uses the first-order method. Thus we need step size $\eta$, while IHS does not require it. As we can see from the above, $\eta=\frac{1}{2}$ is sufficient. One main advantage of our method is that the time complexity is much lower than that of IHS, since it needs only one step of sketching.  If  SRHT is used as the sketch matrix, the complexity of our method becomes $O(nd\log d+d^3\log d+(nd+\text{poly}(d))\log\frac{1}{\epsilon})$ (see Table 1 for comparison with IHS).
\vspace{-0.1in}
\section{Numerical Experiments}

In this section we present some experimental results of our proposed methods (for convenience, we just use the slower HDpwBatchSGD algorithm for the low precision case). We will focus on the iteration complexity and running time. 
Experiments confirm that our algorithms are indeed faster than those existing ones. Our algorithms are implemented using CountSketch as the sketch matrix $S\in\mathbb{R}^{s\times n}$ in the step for computing $R^{-1}$. The Matlab code of CountSketch can be found in \cite{wang2015practical}.  
\begin{itemize}
	\item \textbf{HDpwBatchSGD}, {\em i.e. \textbf{Algorithm 2}}. We use the step size as described in Theorem \ref{thm2} (note that we assume that the step size is already known in advance). 
	\item \textbf{pwGradient}, {\em i.e. \textbf{Algorithm 4}}. We set $\eta=\frac{1}{2}$ as the step size.
\end{itemize}
\subsection{Baseline of Experiments}
\vspace{-0.1in}
\begin{table}[h]
	\caption{Summary of Datasets used in the experiments.}
	\centering\small
	\begin{tabular*}{\linewidth}{@{\extracolsep{\fill}}p{0.05\linewidth}p{0.15\linewidth}p{0.1\linewidth}p{0.1\linewidth}p{0.2\linewidth}@{}}
		\toprule
		Dataset& Rows &  Columns  & $\kappa(A)$ & Sketch Size\\
		\midrule
		Syn1 & $10^5$ & 20 & $10^8$ & 1000 \\ 
		Syn2 & $10^5$ & 20 & $1000$ & 1000 \\
		Buzz & $5\times 10^5$ & 77 & $10^8$ & 20000\\
		Year & $5\times 10^5$ & 90 & $3000$ & 20000\\
		\bottomrule
	\end{tabular*}
\end{table}
	In both the low and high precision cases, we select some widely recognized algorithms with guaranteed time complexities for comparisons. For the low precision case, we choose pwSGD \cite{yang2016weighted}, which has the best known time complexity, and use the optimal step size. We also choose SGD and Adagrad for comparisons. For the high precision case, we use a method called pwSVRG for comparison, which uses preconditioning first and then performs SVRG with different batch sizes (the related method can be found in \cite{rokhlin2008fast}).  Note that since the condition number of the considered datasets are very high, directly using SVRG or related methods could lead to rather poor performance; thus we do not use them for comparison (although  \cite{pmlr-v70-tang17a} used SAGA for comparison, it was done after normalizing the datasets).  We also compare  our proposed  pwGradient algorithm with IHS. The code for SGD and Adagrad can be found in (https://github.com/hiroyuki-kasai/SGDLibrary). All the methods are implemented using MATLAB. We measure the performance of methods by the wall-clock time or iteration number. For each experiment, we test  every method 10 times and take the best. Also for the low precision solvers, we firstly normalize the dataset.\par

	The y-axis of each plot  is the relative error $\frac{\|Ax_t-b\|^2_2-\|Ax^*-b\|^2_2}{\|Ax^*-b\|^2_2}$ in the low precision case and the log relative error $\log(\frac{\|Ax_t-b\|^2_2-\|Ax^*-b\|^2_2}{\|Ax^*-b\|^2_2})$ in the high precision case.  Table 3 is a summary of the  datasets and sketch size used in the experiments. The datasets Year\footnote{https://archive.ics.uci.edu/ml/datasets/yearpredictionmsd} and Buzz\footnote{https://archive.ics.uci.edu/ml/datasets/Buzz+in+social+media+} come from UCI Machine Learning Repository \cite{Lichman:2013}.\par
	
	We consider both the unconstrained case and the constrained cases with $\ell_1$ and $\ell_2$ norm ball constraints. For the constrained case, we first generate the optimal solution for the unconstrained case, and then set it as the radius of balls.
	
\subsection{Experiments on Synthetic Datasets}
\begin{figure}[h]
	\centering
	\includegraphics[width=0.5\textwidth, height=3.5cm]{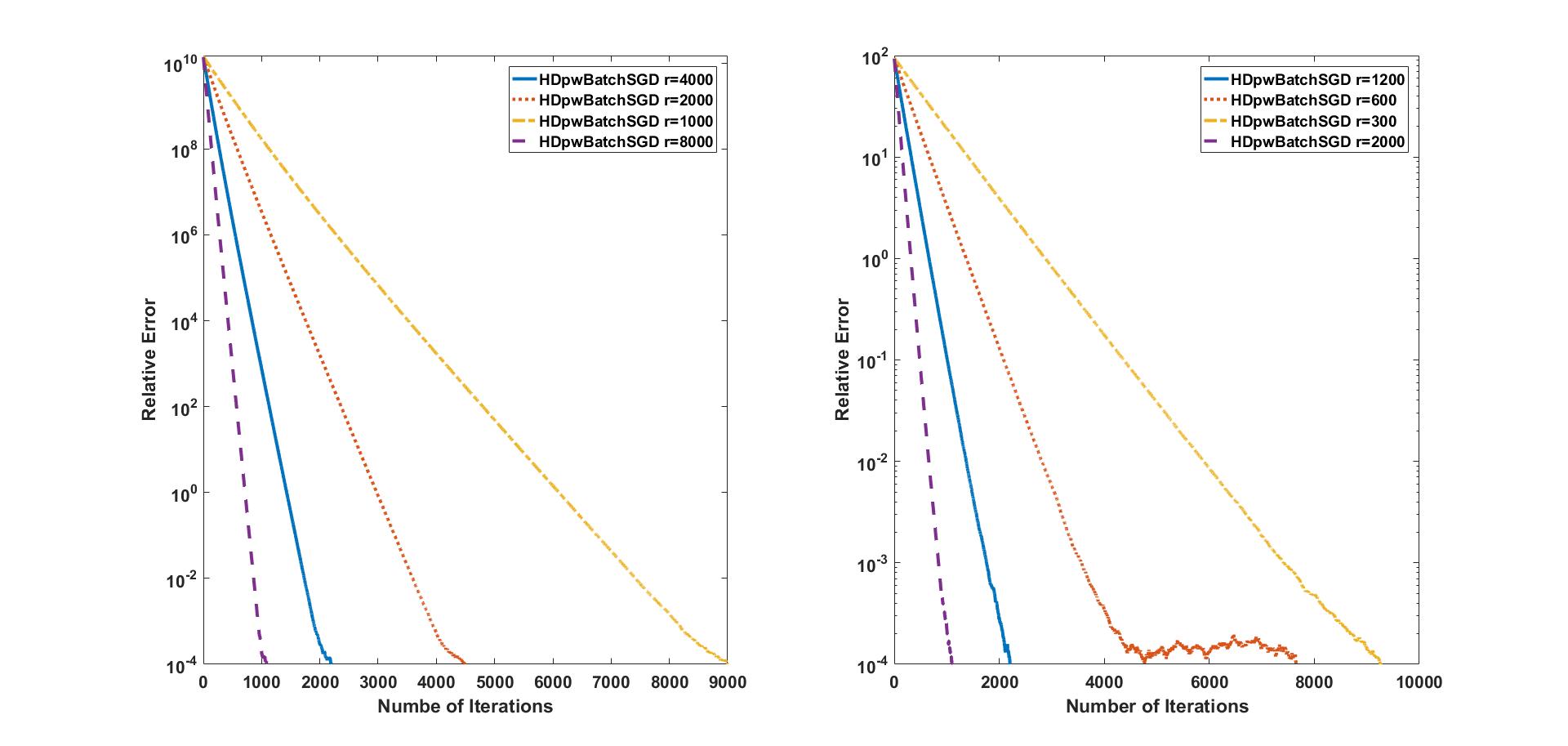}
	\caption{Iteration number of HDpwBatchSGD with different batch size $r$ on  (from left to right)  datasets Syn1 and Syn2 (unconstrained case).}
	\vspace{-0.1in}
\end{figure}
\begin{figure}[h]
	\centering
	\includegraphics[width=0.5\textwidth, height=3.5cm]{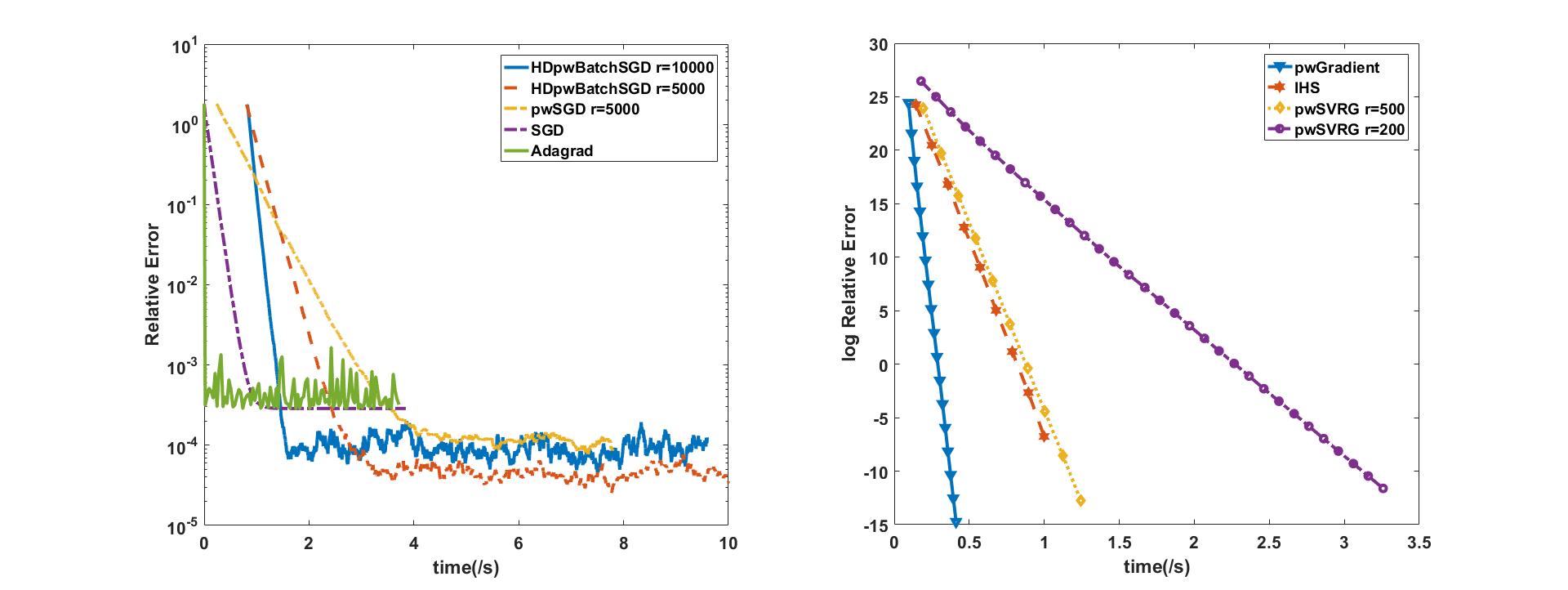}
	\caption{Experimental results on dataset Syn1  (unconstrained case); left is for the low precision solvers, right is for the high precision solvers.}
	\vspace{-0.1in}
\end{figure}
\begin{figure*}[h]
	\centering
	\includegraphics[width=1\textwidth, height=4cm]{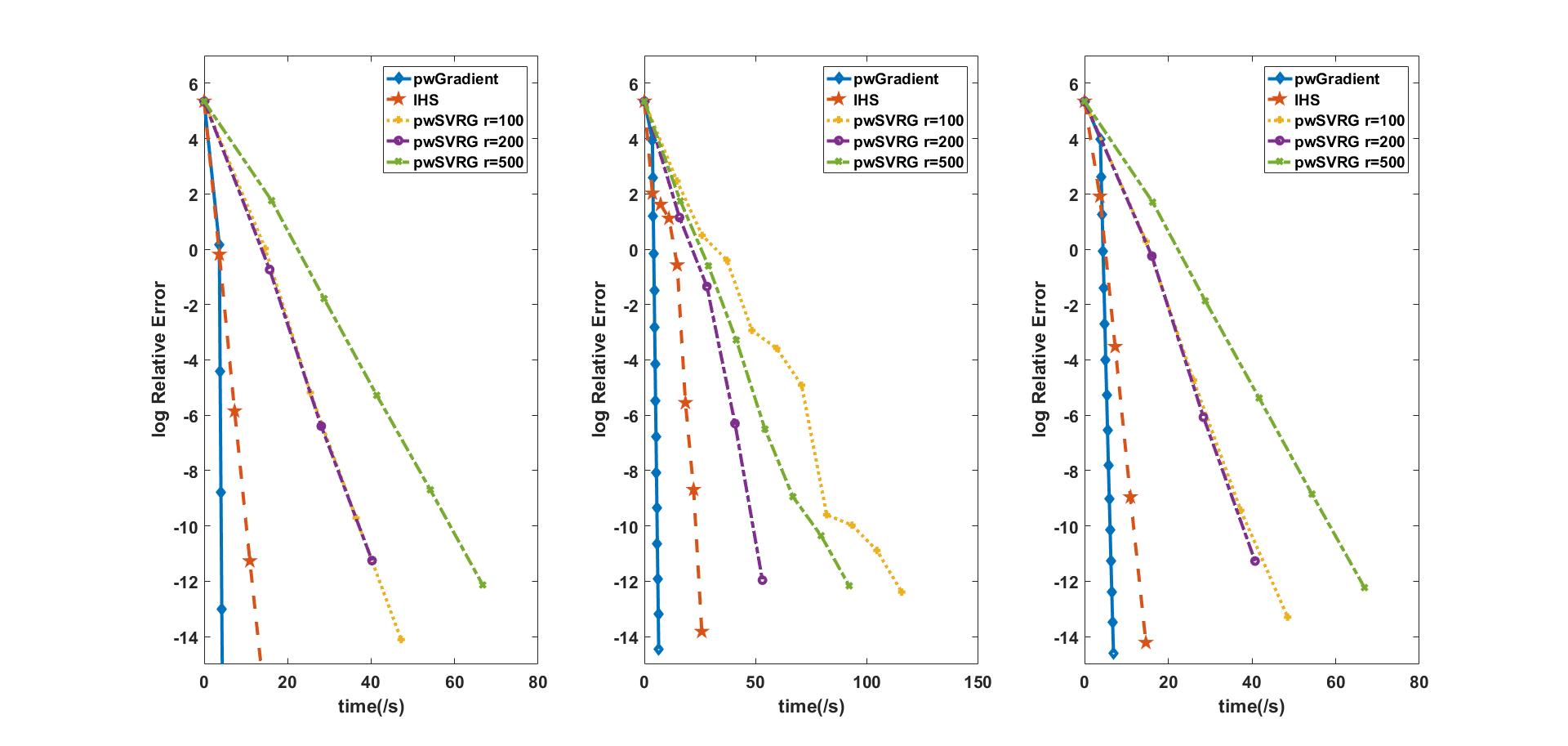}
	\caption{Experimental Results on dataset Year  for the high precision solvers; left is for the unconstrained case, middle is for the $\ell_1$ constrained case, and right is for the $\ell_2$ constrained case.}
\end{figure*}
 We generate a Gaussian vector $x^*$ as the response vector and  let $b=Ax^*+e$, where $e$ is a Gaussian noise with standard variance of $0.1$. In each experiment,  the initial value $x_0$ is set to be the zero vector. We start with some numerical experiments to gain insights to the iteration complexity and relative error shown in \textbf{Theorem \ref{thm4}}, and verify them for the unconstrained case using  synthetic datasets Syn1 and Syn2.
 Later we determine the relative errors for the low and high precision solvers, and plot the results in Figures 1 and 2. 

\subsection{Experiments on Real Datasets}
We consider the unconstrained and the $\ell_1$ and $\ell_2$ constrained linear regression problems on the Buzz dataset for both the low and high precision cases and  on the Year dataset for the high precision case. The results are plotted in Figures 3,  4,  5 and  6, respectively.
\begin{figure}[h]
	\centering
	\includegraphics[width=0.5\textwidth, height=3.5cm]{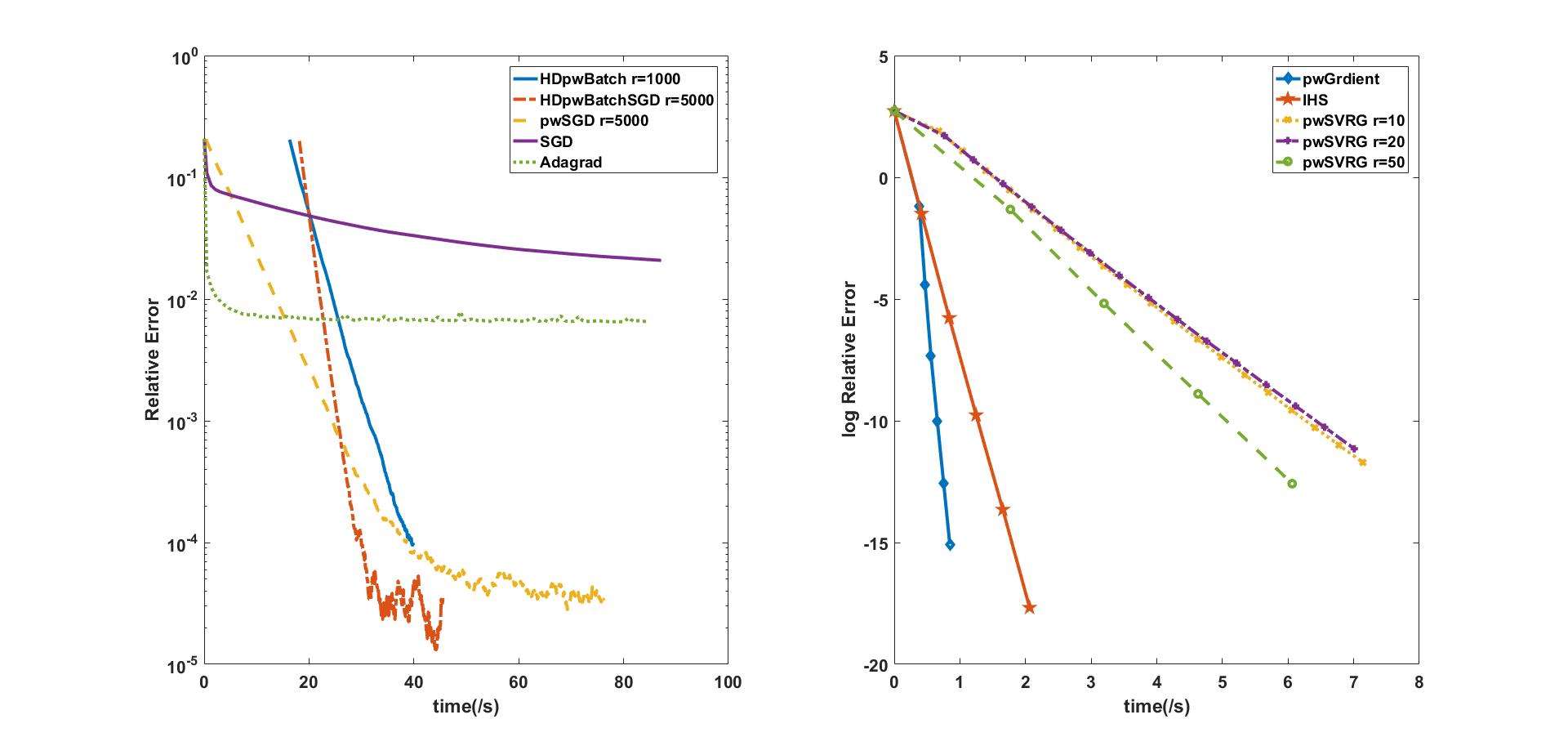}
	\caption{Experimental results on dataset Buzz (unconstrained case); left is for the low precision solvers and right is for the high precision solvers.}
	\vspace{-0.1in}
\end{figure}

 \begin{figure}[h]
	\centering
	\includegraphics[width=0.5\textwidth,height=3.5cm]{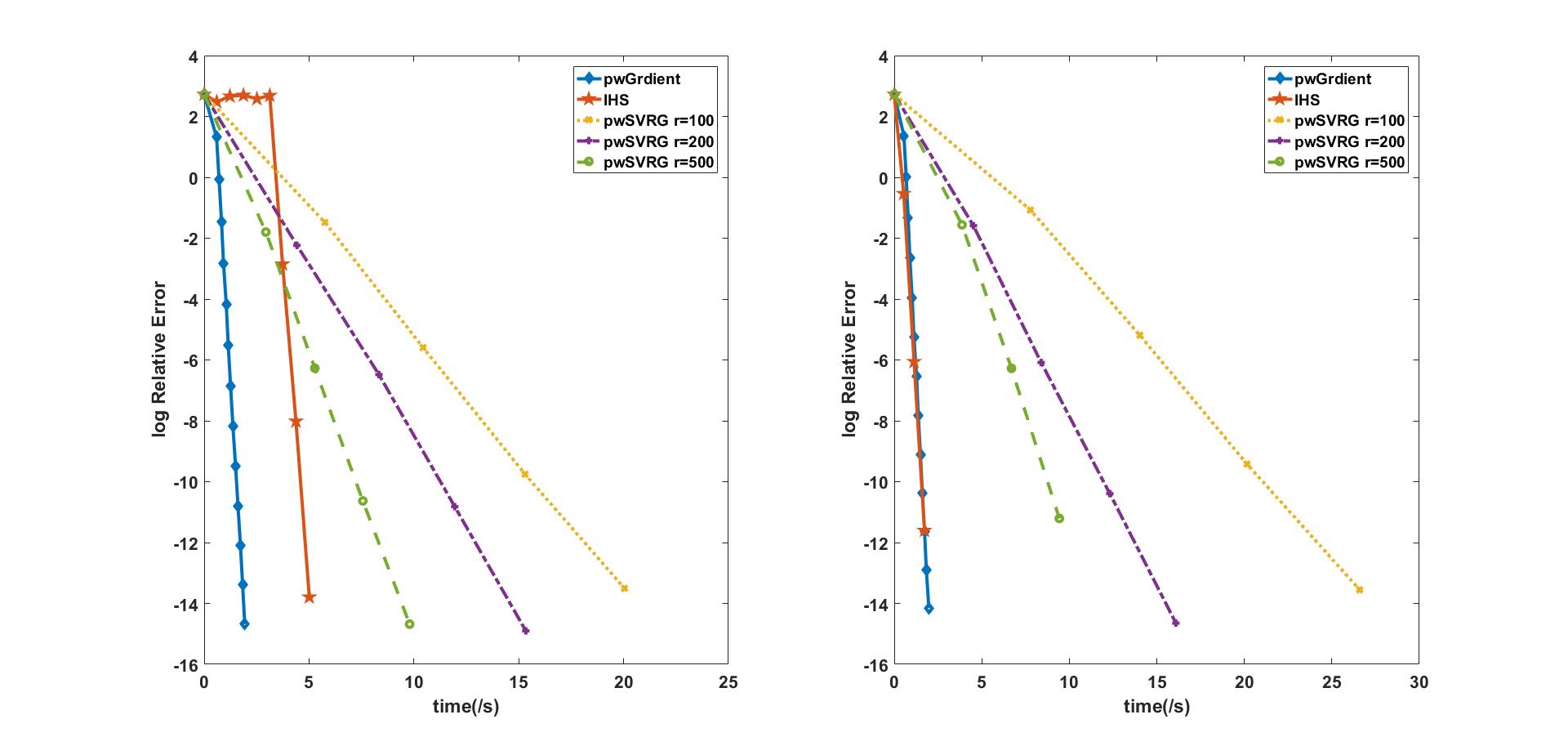}
	\caption{Experimental results on dataset Buzz for the high precision solvers (constrained case); left is for the $\ell_1$ constrained case and right is for the $\ell_2$ constrained case.}
	\vspace{-0.1in}
\end{figure}

\begin{figure}[h]
	\centering
	\includegraphics[width=0.5\textwidth, height=3.5cm]{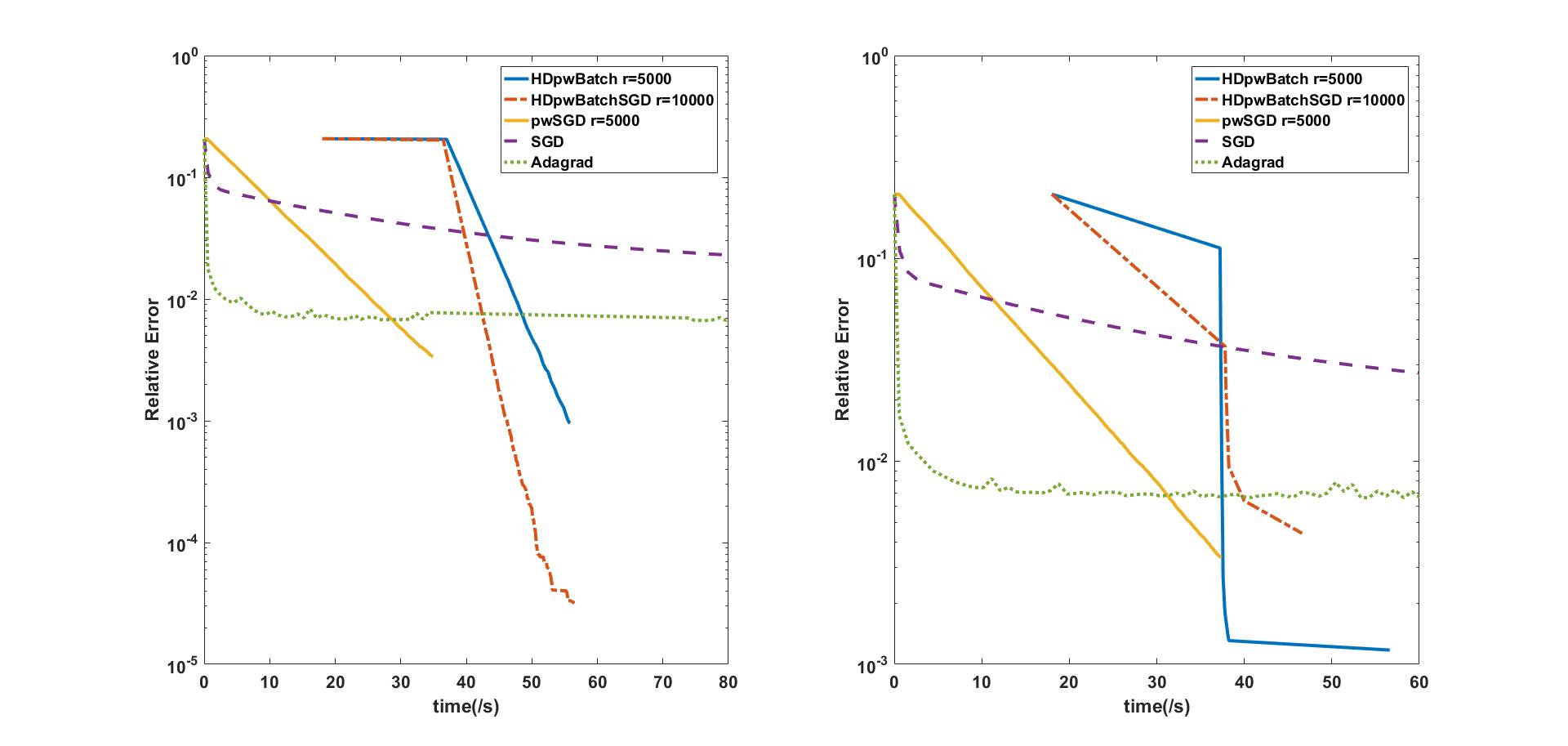}
	\caption{Experimental results on dataset Buzz for the low precision solvers (constrained case); left is for the $\ell_1$ constrained case and right is for the $\ell_2$ constrained case.}
\end{figure}

\subsection{Results}
From Figure 1, we can see that in both cases if  the batch size is increased by a factor of  $b=2$, the iteration complexity approximately decreases by a factor of $b=2$. This confirms the theoretical guarantees of our methods. From other figures, we can also see that in both the high (Figures 2,3,4, and 5) and the low precision (Figures 2,4, and 6) cases, our methods considerably outperform other existing methods. Particularly, in the low precision case, the relative error of HDpwBatchSGD decreases much faster with a large batch size  (except for the $\ell_2$ constrained case in Figure 6).  
%the relative error of HDpwBatchSGD can achieve the decrease more faster( except for the $\ell_2$ constrained case in Figure 6), 
 With a large batch size, our method runs even faster than pwSGD despite a relatively long the preconditioning time (due to the second preconditioning step). This is because in practice CountSketch is faster than SRHT, especially when the dataset is sparse.  This also confirms the claim of our methods.\par
 
 For the high precision case, experiments indicate that pwGradient can even outperform the stochastic methods. Also,
 as we mentioned earlier, pwGradient needs to sketch only once. This enables it to run much faster than IHS, and still preserves the high probability of success.

\section{Conclusion}
In this paper, we studied the large scale constrained linear regression problem, and presented new methods for both the low and high precision cases, using some recent developments in sketching and optimization. For the low precision case, our proposed methods have lower time complexity than the state-of-the-art technique. For the high precision case, our method considerably improves the  time complexity of the Iterative Hessian Sketch method. Experiments on synthetic and benchmark datasets confirm that our methods indeed run much faster than the existing ones.

\appendix

\begin{lemma}{(Lipschitz Tail Bound\cite{ledoux1997talagrand})}\label{lemma:1}
	Let $f$ be a convex function on vectors having $L$-Lipschitz property, and $\epsilon$ be a Rademacher vector. 
	Then for any $t\geq 0$, the following inequality holds 
	%let $\epsilon$ be a Rademacher vector, then we have 
	$$\mathbf{P}\{f(\epsilon)\geq \mathbf{E}f(\epsilon)+Lt\}\leq e^{-\frac{t^2}{8}}.$$
\end{lemma}
By lemma \ref{lemma:1}, we can prove our theorem 1.\par
\begin{theorem}
	Let $HD$ be a Randomized Hadamard Transform, and $U\in \mathbb{R}^{n\times d}$ be an ($\alpha$,$\beta$,2)-conditioned matrix.  Then, the following holds for any constant $c >1$ 
	\begin{equation}
	P\{\max_{i=1,2,\dots,n}\|(HDU)_i\|_2 \geq (1+\sqrt{8\log(cn)})\frac{\alpha}{\sqrt{n}}\} \leq \frac{1}{c}.
	\end{equation}
\end{theorem}
\begin{proof}
	The proof follows similar arguments in \cite{tropp2011improved} for orthogonal matrices. Consider a fixed row index $j\in \{1,2,\dots,n\}$.  Let
	$f(x)=\|e_j^TH\text{diag}(x)U\|_2=\|x^T\text{diag}(e_j^TH)U\|_2$. Then $f(x)$ is convex and $\|f(x)-f(y)\|\leq \|x-y\|_2\|\text{diag}(e^TH)\|_2 \|U\|_2\leq \frac{\alpha}{\sqrt{n}}\|x-y\|_2$, since each entry of $H$ is either $\frac{-1}{\sqrt{n}}$ or $\frac{1}{\sqrt{n}}$. Thus $f(x)$ is $\frac{\alpha}{\sqrt{n}}$-Lipschitz. 
	By the fact that $f(\epsilon)$ is a Rademacher function, we have 
	\begin{align*}
	&\mathbf{E}f(\epsilon)\leq[\mathbf{E}f^2(\epsilon)]^{\frac{1}{2}}
	=(\mathbf{E}\|\epsilon^T\text{diag}(e_j^TH)U\|_2^2)^{\frac{1}{2}}\\
	&=\|\text{diag}(e_j^TH)U\|_F\leq \|\text{diag}(e_j^TH)\|_2\|U\|_F\leq \frac{\alpha}{\sqrt{n}}.
	\end{align*}
	
	Then by {\bf Lemma 1}, taking $t=\sqrt{8\log(cn)}$ and the union for all arrow indices, we have the theorem.
\end{proof}

\begin{lemma}\label{lemma:3}
	After two steps of preconditioning as in (7), (8), the following holds with high probability (approximately $0.9$) for the  stochastic optimization problem: $\min g(y)=E_{i\sim \mathcal{D'}}g_{{i}}(y)$, where $g_{{i}}(y)=n\|(HDU)_iy-(HDb)_i\|_2^2$, $g(y)=\|HDUy-HDb\|_2^2$, 
	$i\sim  \mathcal{D'}$ is uniformly sampled from $\{1,2,\dots,n\}$, and $U$ is an $(\alpha,\beta,2)$-conditioned basis of $A$.
	%the following holds with high probability (approximately $0.9$), 
	\begin{align}
	&\mu\geq \frac{2}{\beta^2},\\
	&\sup\|(HDU)_{i}\|^2_2\leq \alpha^2(1+\sqrt{8\log 10n})^2,\\
	&\sigma^2\leq 4\alpha^2(1+\sqrt{8\log(10n)})^2\sup_{y\in \mathcal{W}'}g(y)\\
	&=4\alpha^2(1+\sqrt{8\log(10n)})^2\sup_{x\in \mathcal{W}}\|Ax-b\|^2
	\end{align}
	where the constant $10$  comes from  Theorem 2 with $c=10$.
\end{lemma}

\begin{proof}
	We know $\mu=2\sigma_{min}^2(HDU)=2\sigma_{min}^2(U)\geq \frac{2}{\beta^2}$. Since $U$ is an $(\alpha,\beta,2)$-conditioned basis of $A$.
	By {\bf Theorem 2}, we know that with probability at least 0.9, the norm of each row of $HDU$ is smaller than $\frac{\alpha}{\sqrt{n}}(1+\sqrt{8\log 10n})$. Hence, we have 
	$$\sup\|(HDU)_{\tau_i}\|^2_2\leq \frac{\alpha^2}{n}(1+\sqrt{8\log cn})^2.$$ 
	For $\sigma^2=\sup_{y\in \mathcal{W'}}\mathbb{E}_{i\sim D}\|\nabla g_{{i}}(y)-\nabla g(y)\|_2^{2}$, we have the following with probability at least 0.9,
	\begin{align*}
	&\sigma^2=\sup_{y\in \mathcal{W}'}\mathbb{E}_{j\sim \mathcal{D}}\|2n(HDU)_j^T((HDU)_jy-(HDb)_j)\|_2^2\\
	&-\|2(HDU)^T((HDU)y-HDb)\|_2^2\\
	&=4n\sum_{j=1}^{n}\|((HDU)_j^T((HDU)_jy^*-(HDb)_j))\|_2^2\\
	&-4\|(HDU)^T((HDU)y-HDb)\|_2^2\\
	&\leq 4n\sup\|(HDU)_{{j}}\|^2_2 \|HDUy^*-HDUb\|^2_2\\
	&-\sigma_{min}(HDU)^2\|((HDU)y-HDb)\|_2^2\\
	&\leq 4n\frac{\alpha^2}{n}(1+\sqrt{8\log(10n)})^2\sup_{y\in \mathcal{W'}}g(y).
	\end{align*}
\end{proof}
where the last inequality comes from Theorem 2 and $\sigma_{min}(HDU)=O(1)$.
With Lemma \ref{lemma:3} and Theorem 2, we can now show our main theorem.
\begin{theorem}\label{thm4}
	Let  $A$ be a matrix in $\mathbb{R}^{n\times d}$ and $b$ be a vector in $\mathbb{R}^{d}$.  Let $f(x)$ denote $\|Ax-b\|^2_2$. Then with some fixed step size $\eta$, we have 
	\begin{equation}
	\mathbb{E}{f(x_T^{avg})-f(x^*)}\leq \frac{3\sqrt{2}D_{\mathcal{W}}\sigma}{\sqrt{rT}},
	\end{equation}
	where $\sigma^2\leq O(d\log(n)\sup_{x\in\mathcal{W}}\|Ax-b\|_2^2)$.
	After $T= \Theta(\frac{d\log{n}}{r\epsilon^2})$ iterations with some step size $\eta$, {\bf Algorithm 2} ensures the following with high probability, we have $\mathbb{E}{f(x_T)-f(x^*)}\leq \epsilon$. 
	
\end{theorem}
\begin{proof}
	Consider  $\{y_i\}_{i=1}^{T}$ updated by Lemma \ref{lemma:3} with $y_0=Rx_0$.  We first show by mathematical induction that  $y_i=Rx_i$ and $y^*=Rx^*$ for all $i$.  
	Clearly, by the definition of $y_i$, this is true for $i=0$. Assume that it is true for $i=k$. In the $(k+1)$-th iteration,  we assume that the i-th sample in the k-th iteration is obtained by using SGD. Then, we have (denote $m=\frac{n}{r}$)
	\begin{multline*}
	y_{k+1}=\arg\min_{y\in \mathcal{W'}} \eta\langle\nabla g_{\tau_{k}}(y_k),y\rangle+\frac{1}{2}\|y-y_k\|_2^2\\
	=\arg\min_{y\in \mathcal{W'}} 2m\eta\sum_{j\in{\tau_{k}}}((HDU)_jR^{-1}x_k-(HDb)_j)\\(HDA)_j^T R^{-1}y
	+\frac{1}{2}\|y-Rx_t\|^2_2.
	\end{multline*}
	From Steps 5 and 6, we know that 
	\begin{multline*}
	x_{k+1}=\arg\min_{x\in \mathcal{W}} 2m\eta\sum_{j\in{\tau_{k}}}((HDA)_jx_k-(HDb)_j)\\(HDA)_jx+\frac{1}{2}\|Rx-Rx_t\|^2_2 ,
	\end{multline*}
	where $\mathcal{W}=R^{-1}\mathcal{W'}$.  From above, we know that $x_{k+1}=R^{-1}y_{k+1}$. This means that $y_i=Rx_i$ is true for all $i$.
	Next, by using the variance in the mini-batch SGD, and lemma 3, we know the $\sigma^2_{batch}=\frac{\sigma^2}{r}$, where $\sigma^2$ is as in Lemma \ref{lemma:3}. Then by Theorem 2 we get $\mathbb{E}{g(y_T^{\text{avg}})-y(y^*)}\leq \frac{3\sqrt{2}D_{\mathcal{W'}}\sigma}{\sqrt{rT}}$, replacing $y_T=Rx_T, y_T^{avg}=x_T^{avg}$ and $D_{\mathcal{W'}}=D_{\mathcal{W}}$(by the definition), we get the proof.
\end{proof}

\begin{theorem}
	Let $f(x)=\|Ax-b\|^2$.  Then, for some step size $\eta=O(1)$ in pwGradient, the following holds,  
	%\begin{equation}
	$$f(x_t)-f(x^*)\leq (1-O(1))^tO((f(x_0)-f(x^*))), $$
	%\end{equation}
	with high probability.
\end{theorem}

\begin{proof}
	Similar to the proof of Theorem 4, we can show that  the updating step is just performing the projected gradient descent operations on $\|AR^{-1}y-b\|_2^2$ and thus $x_{t+1}=R^{-1}y_{t+1}$. Since the condition number of $U=AR^{-1}$ is $O(1)$,  by the convergence rate of gradient descent on strongly convex functions \cite{nesterov2013introductory} and with step size $\eta=O(1)$, we know that 
	\begin{align*}
	f(x_t)-f(x^*)=&\|Uy_{t}-b\|_2^2-\|Uy^*-b\|_2^2\\
	&\leq \frac{2\sigma^2_{\max}(U)}{2}(1-O(1))^{2k}\|y_0-y^*\|_2^2.
	\end{align*} 
	By the strongly convexity property, we know that $\frac{2\sigma^2_{\min}(U)}{2}\|y_0-y^*\|_2^2\leq \|Uy_0-b\|_2^2-\|Uy^*-b\|_2^2=f(x_0)-f(x^*)$.  Also, by $\kappa{(U)}=O(1)$ (see Table 2). We get the theorem.
\end{proof}
\appendix
\section{Details on HDpwAccBatchSGD}
For the completeness of the paper, we first give the preliminaries on accelerated stochastic gradient descent and multi-epoch accelerated stochastic gradient descent method, more can refer to \cite{ghadimi2013optimal}.
Now we consider the uniform sampling version,
\begin{equation}\label{eq:5}
\min_{x\in\mathcal{W}}F(x)=\frac{1}{n}\sum_{i=1}^{n}{f_i(x)}.
\end{equation}
We assume $F(\cdot)$ is $\mu$-strongly convex and $L$-smooth. In the t-th iteration, the accelerated stochastic gradient descent is the following:
\begin{align}
&\tilde{x}_t=(1-q_t)\hat{x}_{t-1}+q_tx_{t-1},\\
&x_t=\arg\min_{x\in \mathcal{W}}\{\eta_t[\langle \nabla f_i(\tilde{x}_t),x\rangle+\frac{\mu}{2}\|\tilde{x}_t-x\|_2^2]+\frac{1}{2}\|x-x_{t-1}\|_2^2\}\\
&\hat{x}_t=(1-\alpha_t)\hat{x}_{t-1}+\alpha_tx_t
\end{align}
where the step size $\{q_t\},\{\alpha_t\},\{\eta_t\}$ are later defined and the initial points satisfy $\hat{x}_0=\tilde{x}_0=x_0$.
\begin{algorithm}[h]
	\caption{Multi-epoch Stochastic Accelerated Gradient Descent} 
	$\mathbf{Input}$: 	 $p_0\in \mathcal{W}$ is the initial point, and a bound $V_0$ such that $F(x_0)-F(x^*)\leq V_0$ is given, $S$ is the epoch number.
	\begin{algorithmic}[1]
		\For{$s=0,1,\cdots S$}{
			\State Run $N_s$ iterations of Stochastic Accelerated Gradient method with $x_0=p_{s-1}, \alpha_t=\frac{2}{t+1}, q_t=\alpha_t, $ and $\eta_t=\eta_st$, where 
			\begin{align*}
			&N_s=\max\{4\sqrt{\frac{2L}{\mu}},\frac{64\sigma^2}{3\mu V_02^{-s}}\}\\
			&\eta_s=\min\{\frac{1}{4L},\sqrt{\frac{3V_02^{-(s-1)}}{2\mu\sigma^2N_s(N_s+1)^2}}\}
			\end{align*}
			\State Set $p_s=\hat{x}_{N_s}$, where $\hat{x}_{N_s}$ is get from step 1.
			\EndFor}\\
		\Return $p_s$
	\end{algorithmic}
\end{algorithm}
\begin{algorithm}[h]
	\caption{HDpwAccBatchSGD($A$,$b$,$x_0$,$r$,$s$, $V_0$,$S$)} 
	$\mathbf{Input}$: 	 $x_0$ is the initial point, $r$ is the batch size, $\eta$ is the fixed step size, $s $ is the sketch size, $S$ is the number of epochs and bound $V_0$ satisfies $F(x_0)\leq V_0$.
	\begin{algorithmic}[1]
		\State Compute $R\in \mathbb{R}^{d\times d}$ which makes $AR^{-1}$ an $O(\sqrt{d}),O(1),2)$-conditioned basis of $A$ as in Algorithm 1 by using a sketch matrix $S$ with size $s \times n$.
		\State  Compute $HDA$ and $HDb$, where $HD$ is a Randomized Hadamard Transform.
		\State Run Multi-epoch Stochastic Accelerated Gradient Descent with $x_0$, $V_0$, $L=O(1)$,$\mu=O(1)$ and with $S$ epochs,
		where the step of Stochastic Accelerated Gradient Descent is as following:
		\State Randomly sample  an indices set $\tau_{t}$ of size $r$, where each index in $\tau_{t}$ is i.i.d uniformly sampled.
		\State 
		$c_{\tau_{t}}=2\frac{n}{r}\sum_{j \in \tau_t}{(HDA)_j^T[(HDA)_jx_{t-1}-(HDb)_j]} 
		=2\frac{n}{r} (HDA)_{\tau_t}^T[(HDA)_{\tau_{t}}x_{t-1}-(HDb)_{\tau_{t}}]  $  
		\begin{align*}
		&\tilde{x}_t=(1-q_t)\hat{x}_{t-1}+q_tx_{t-1},\\
		&x_t=\arg\min_{x\in \mathcal{W}}\{\eta_t[\langle c_{\tau_{t}} ,x\rangle+\frac{\mu}{2}\|R(\tilde{x}_t-x)\|_2^2]\\
		&+\frac{1}{2}\|R(x-x_{t-1})\|_2^2\}\\
		&\hat{x}_t=(1-\alpha_t)\hat{x}_{t-1}+\alpha_tx_t
		\end{align*}
		\Return $p_s$
	\end{algorithmic}
\end{algorithm}
\begin{theorem}
	If Assumption 1 and 2 hold, the after $O(\sqrt{\frac{L}{\mu}}\log(\frac{V_0}{\epsilon})+\frac{\sigma^2}{\mu\epsilon})$ iterations of stochastic accelerated gradient descent with $O(\log(\frac{V_0}{\epsilon}))$ epochs, the output of multi-epoch stochastic accelerated gradient descent $x_S$ satisfies $\mathbb{E}F(x_S)-F(x^*)\leq \epsilon$, here $V_0$ is a given bound such that $F(x_0)-F(x^*)\leq V_0$.
\end{theorem}
Now we introduce our algorithm \textbf{HDpwAccBatchSGD},

we have the following theorem,
\begin{theorem}
	Denote $F(x)=\|Ax-b\|_2^2$, fix $\epsilon<V_0$, then for HDpwAccBatchSGD, after $O(\log(\frac{V_0}{\epsilon})+\frac{d\log n}{r\epsilon})$ iterations of stochastic accelerated gradient descent with $S=O(\log(\frac{V_0}{\epsilon}))$ epochs, the output of multi-epoch stochastic accelerated gradient descent $x_S$ satisfies $\mathbb{E}F(x_S)-F(x^*)\leq \epsilon$.
\end{theorem}
\begin{proof}
	The proof is similar with Lemma \ref{lemma:3} and Theorem 4.
	
\end{proof}

\bibliography{AAAI}
\bibliographystyle{aaai}
\end{document}